\documentclass[oneside,a4paper]{article}

\usepackage{bbding,ifsym,pifont}
\usepackage{booktabs,hhline,enumerate}
\usepackage{graphicx,color,subfigure}

\usepackage{bm,subfloat,url}

\usepackage{times}
\usepackage{epsfig}
\usepackage{graphicx}
\usepackage{amsmath,amsthm}
\usepackage{amssymb}
\usepackage{enumerate}
\usepackage[algoruled,vlined]{algorithm2e}
\usepackage{mathtools}
\usepackage{color}
\usepackage{tabu}

\usepackage{multirow}
\newtheorem{assumption}{Assumption}

\usepackage{enumerate,subfigure}
\newtheorem{lemma}{Lemma}

\newtheorem{remark}{Remark}
\newtheorem{theorem}{Theorem}
\newtheorem{definition}{Definition}
\newtheorem{corollary}{Corollary}
\newtheorem{proposition}{Proposition}

\usepackage{xypic}

\begin{document}

\begin{center}
	{\Large A Differential Topological View of 
	Challenges in Learning \\[1mm] with Feedforward Neural Networks
	} \\[6mm]
	{\large Hao~Shen} \\[0.5mm]
	E-mail: hao.shen@fortiss.org \\[2mm]
	fortiss -- Landesforschungsinstitut des Freistaats Bayern, Germany \\[7mm]
\end{center}

\begin{abstract}
	Among many unsolved puzzles in theories of Deep Neural 
	Networks (DNNs), there are three most fundamental 
	challenges that highly demand solutions, namely, 
	\emph{expressi\-bility}, \emph{optimisability}, 
	and \emph{generalisability}.
	Although there have been significant progresses
	in seeking answers using various theories, e.g. 
	information bottleneck theory, sparse representation, 
	statistical inference, Riemannian geometry, etc., so far
	there is no single theory that is able to provide solutions to all 
	these challenges.
	In this work, we propose to engage the theory of 
	\emph{differential topology} to address the three 
	problems.
	By modelling the dataset of interest as a smooth manifold, DNNs 
	can be considered as compositions of smooth maps between smooth manifolds.
	Specifically, our work offers a differential topological view of 
	\emph{loss landscape of DNNs}, 
	\emph{interplay between width and depth in expressibility}, and 
	\emph{regularisations for generalisability}.
	Finally, in the setting of deep representation learning, we 
	further apply the \emph{quotient topology} to investigate the architecture 
	of DNNs, which enables to capture nuisance factors in data 
	with respect to a specific learning task.
\end{abstract}

\begin{center}
	\textbf{\small{Index Terms}} \\[2mm]
	Deep Neural Networks (DNNs), 
	expressibility, 
	optimisability, 
	generalisability,
	differential topology,
	quotient topology.
\end{center}

\section{Introduction}\label{sec:01}
Recently, Deep Neural Networks (DNNs) have attracted enormous 
research attentions, due to their prominent performance comparing 
to various state of the art approaches in pattern recognition, 
computer vision, and speech recognition 
\cite{bish:book96,lecu:nature15,yudo:book15}.
Despite vast experimental evidences, such success of DNNs is still 
not theoretically understood.
Among many unsolved puzzles, there are three most 
fundamental challenges in research of DNNs
that highly demand solutions, namely, 
\emph{expressibility}, \emph{optimisability}, and 
\emph{generalisability}.
Although there have been significant progresses in searching for 
answers using various theories, e.g. function approximation, information 
bottleneck principle, sparse representation, statistical inference, Riemannian 
geometry, etc., a complete solution to the overall puzzle is still missing.
In this work, we address the three grand challenges in the
framework of Feedforward Neural Networks (FNNs) from the perspective 
of differential topology.

Although classic results have already proven that a shallow FNN with only one 
hidden layer having an unlimited number of units
is a universal approximator of continuous functions on compact subsets
of $\mathbb{R}^{m}$ \cite{horn:nn91},
recent extensive practice suggests that deep FNNs are 
more expressive than their shallow counterparts \cite{beng:lskm07}.
Such an observation has been confirmed by showing that 
there are functions, which are expressible by a width-bounded deep FNN, 
but require \emph{exponentially} many neurons in the hidden layers of a 
shallow FNN for a specified accuracy \cite{elda:colt16,telg:colt16}.
The work in \cite{roln:iclr18} further shows that the total number 
of neurons required to approximate natural classes of multivariate 
polynomials grows only \emph{linearly} in deep FNNs, but grows 
\emph{exponentially} in a two-layer FNN.
Meanwhile, the impact of width of DNNs has also been proven to be
critical for gaining more expressive power of ReLU networks
\cite{luzh:nips17}. 
Despite these rich results about expressibility of DNNs, interplay or 
trade-off between depth and width for achieving good performance 
is not yet concluded.

Training DNNs is conventionally considered to be difficult, 
mainly due to its associated optimisation problem being highly 
non-convex \cite{sutt:ccss86,widr:pieee90}.
Recent observation of prominent performance of gradient descent 
based algorithms has triggered enormous interests and efforts in  
characterising loss landscape and global optimality of DNNs 
\cite{kawa:nips16,nguy:icml17,haef:cvpr17,shen:cvpr18}.
Most of these works assume exact fitting of a finite number 
of training samples with a sufficiently large DNN,
and suggest that full rank weight matrices play a criti\-cal role
in ensuring good performance of DNNs.
Curiously, besides these arguments from the optimisation perspective, 
the impact of requiring weight matrices to have full rank
are still not clear to the other two challenges.

Arguably, generalisability is the most puzzling mystery of DNNs 
\cite{lawr:aaai97,zhan:iclr17}.
There have been many recent efforts dedicated to explain this 
phenomenon, such as deep kernel learning \cite{belk:icml18},
information bottleneck \cite{tish:itw15,saxe:iclr18,achi:jmlr18}, and 
classification bound analysis \cite{soko:tsp17}.
Many heuristic mechanisms have also been developed 
to enhance generalisability of DNNs, e.g.
dropout regularisation \cite{sriv:jmlr14} and norm-based
control \cite{neys:nips17,soko:tsp17}.
So far, there is no single theory or practice that can provide 
affirmative conclusions about the mysterious generalisability of DNNs.

Most recently, there has been an increasing interest in analysing
DNNs from geometric and topological perspectives, such as algebraic topology 
\cite{suns:aaai16} and Riemannian geometry \cite{haus:nips17}.
Particularly, the work in \cite{fawz:cvpr18} argues that geometric 
and topological properties of state of the art DNNs are crucial for 
better understanding and building theories of DNNs.
It is also worth noticing that geometric and topolo\-gical analysis 
is indeed a classic methodology in the research of neural networks
\cite{mins:book17}.
In this work, we extend such a trend to employ the theory of differential 
topology to study the three challenges of DNNs.

\section{Optimisation of DNNs: Full rank weights}
\label{sec:02}
Let us denote by $L$ the number of layers in a DNN, 
and by $n_{l}$ the number of processing units in 
the $l$-th layer with $l = 1, \ldots, L$.
Specifically, by $l = 0$, we refer it to as the input layer.
Hidden layers in DNNs can be modelled as the following 
parameterised nonlinear map
\begin{equation}
\label{eq:layer}
	\Lambda_{l}(W_{l},\cdot) \colon \mathbb{R}^{n_{l-1}} \to \mathbb{R}^{n_{l}},
	\quad x \mapsto \sigma\big( W_{l}^{\top}x + b_{l} \big),
\end{equation}
where $b_{l} \in \mathbb{R}^{n_{l}}$ is a bias that is treated as
a constant in this work for the sake of simplicity in presentation,
and $\sigma \colon \mathbb{R}^{n_{l}} \to \mathbb{R}^{n_{l}}$ 
applies a unit nonlinear function entry-wise to its input, e.g.
Sigmoid, SoftPlus, and ReLU.
In this work, we restrict activation functions to be \emph{smooth}, 
\emph{monotonically increasing}, and \emph{Lipschitz}.

Now, let us denote by $\phi_{0} \in \mathbb{R}^{n_{0}}$ the input.
We can then define evaluations at all layers as 
$\phi_{l} := \Lambda_{l}(W_{l},\phi_{l-1})$ iteratively.
By denoting the set of all parameter matrices in the DNN by 
%
	$\bm{\mathcal{W}} := \mathbb{R}^{n_{0} \times n_{1}} \times \ldots \times
	\mathbb{R}^{n_{L-1} \times n_{L}}$,
%
we compose all layer-wise maps to define the overall \emph{DNN map} as
\begin{equation}
	\begin{split}
		F \colon \bm{\mathcal{W}} \times \mathbb{R}^{n_{0}} \to &~ \mathbb{R}^{n_{L}}, \\
		(\mathbf{W},\phi_{0}) \mapsto &~ \Lambda_{L}(W_{L},\cdot) \circ
		\ldots \circ \Lambda_{1}(W_{1}, \phi_{0}).
	\end{split}
\end{equation}
Note, that the last layer $\Lambda_{L}(W_{L},\cdot)$ is commonly linear,
i.e., the activation function in the last layer is the identity map 
$\operatorname{id}$.
We define the set of parameterised maps specified 
by a given DNN architecture as
\begin{equation}
\label{eq:mlp}
	\bm{\mathcal{F}}(n_{0},\ldots,n_{L}) \!:=\! \big\{\! F(\mathbf{W},\cdot) \colon 
	\mathbb{R}^{n_{0}} \!\to\! \mathbb{R}^{n_{L}} 
	\big| \mathbf{W} \!\in\! \bm{\mathcal{W}} \big\},\! 
\end{equation}
which specifies the architecture of the DNN, i.e., 
the number of units in each layer.

Many \emph{machine learning} tasks can be formulated as a problem of learning 
a task-specific \emph{ground truth map} (\emph{task map} for short)
$f^{*} \colon \bm{\mathcal{X}} \to \bm{\mathcal{Y}}$,
where $\bm{\mathcal{X}}$ and $\bm{\mathcal{Y}}$ denote an \emph{input space} and 
an \emph{output space}, respectively.
The problem of interest is to approximate $f^{*}$, given only a finite number of
samples in either $\bm{\mathcal{X}}$ or $\bm{\mathcal{X}} \times \bm{\mathcal{Y}}$.
For supervised learning, given only a finite number of samples 
$\{(x_{i},y_{i})\}_{i=1}^{T} \subset \bm{\mathcal{X}} \times \bm{\mathcal{Y}}$ 
with $y_{i} := f^{*}(x_{i})$, one can utilise a DNN $F(\mathbf{W},\cdot) \in
\bm{\mathcal{F}}(n_{0},\ldots,n_{L})$ 
to approximate the task map $f^{*}$, via minimising an empirical total 
loss function that is defined as 
\begin{equation}
\label{eq:total_loss}
	\mathcal{J}(\mathbf{W}) := \frac{1}{T}  \sum\limits_{i=1}^{T} 
	E\big( F(\mathbf{W},x_{i}), y_{i} \big),
\end{equation}
where $E \colon \bm{\mathcal{Y}} \times \bm{\mathcal{Y}} \to \mathbb{R}$ is 
a suitable error function that evaluates the estimate $F(\mathbf{W}, x_{i})$ against 
the supervision $y_{i}$.
Clearly, given only a finite number of samples, the task map
$f^{*}$ is hardly possible to be exactly learned as the solution 
in $\bm{\mathcal{F}}(n_{0},\ldots,n_{L})$.
Nevertheless, exact learning of a finite number of samples
is still of theoretical interest.
\begin{definition}[Exact DNN approximator]
	Given a DNN architecture $\bm{\mathcal{F}}(n_{0},\ldots,n_{L})$, and
	let $f^{*} \colon \bm{\mathcal{X}} \to \bm{\mathcal{Y}}$ be the task map.
	Given samples $\{x_{i},y_{i}\}_{i=1}^{T} \subset \bm{\mathcal{X}} \times
	\bm{\mathcal{Y}}$, 
	a DNN map $F(\mathbf{W}, \cdot) \in \bm{\mathcal{F}}(n_{0},\ldots,n_{L})$, 
    which satisfies $F(\mathbf{W}, x_{i}) = f^{*}(x_{i})$
    for all $i = 1, \ldots,T$, 
    %
	%
	is called an \emph{exact DNN approximator} of $f^{*}$ with respect to
	the $T$ samples. 
\end{definition}

\noindent In order to ensure its attainability and uniqueness via an optimisation 
procedure, we adopt the following assumption as a practical principle 
of choosing the error function.

\begin{assumption}
\label{ass:error}
	For a given $y \in \bm{\mathcal{Y}}$, the error function $E( \phi_{L}, y )$ 
	is differentiable with respect to its first argument.
	Existence of global minima of $E$ is guaranteed, and $\phi_{L} = y \in 
	\bm{\mathcal{Y}}$ is a global minimum of $E$,
	\emph{if and only if} the gradient of $E$ with respect to the first 
	argument vanishes at $\phi_{L} = y$, i.e., $\nabla_{\!E}(\phi_{L}) = 0$.
\end{assumption}
\begin{remark}
	Assumptinon~\ref{ass:error} guarantees the existence of global minima 
	of the error function $E$. Since the summation in the empirical total
	loss is finite, the function value of $\mathcal{J}$ has a finite
	lower bound.
	Furthermore, it also ensures a global minimiser of the total 
	loss function $\mathcal{J}$, if exists, to coincide with the exact 
	learning of a finite set of samples. 
	Popular choices of the error function \cite{hart:book04}, such as the classic squared loss, 
	smooth approximations of $\ell_{p}$ norm with $0 < p < 2$, Blake-Zisserman loss, 
	and Cauchy loss, satisfy this assumption.
\end{remark}

Let $\dot{\sigma}_{l}(x) \in \mathbb{R}^{n_{l}}$ be the vector of the 
derivative of the activation function in the $l$-th layer, and 
we define a set of diagonal matrices as $\Sigma_{l}^{'}(x_{i}) := 
\operatorname{diag}(\dot{\sigma}_{l}(x_{i}))$ for all $l = 1, \ldots, L$.
We further construct a sequence of matrices as
\begin{equation}
\label{eq:psi}
	\Psi_{l}(x_{i}) := \Sigma_{l}^{'}(x_{i}) W_{l+1} \Psi_{l+1}(x_{i}) 
	\in \mathbb{R}^{n_{l}\times n_{L}}, 
\end{equation}
for all $l = L-1,\ldots,1$ with $\Psi_{L}(x_{i}) = \Sigma_{L}^{'}(x_{i}) 
\in \mathbb{R}^{n_{L}\times n_{L}}$.
Then, the Jacobian matrix of the DNN map $F(\mathbf{W},\cdot)$ with 
respect to the weight $\mathbf{W}$ can be presented as 
\begin{equation}
\label{eq:dnn_jaco}
	\operatorname{D}_{1}\!F(\mathbf{W},x_{i}) =\!\!
	\left[\!\! \begin{array}{rcl}
	\Psi_{L}(x_{1}) & \!\!\!\!\!\otimes\!\!\!\!\!\! & \phi_{L-1}(x_{i}) \\
	& \!\!\!\!\!\vdots\!\!\!\!\! & \\
	\Psi_{1}(x_{1}) & \!\!\!\!\!\otimes\!\!\!\!\!\! & \phi_{0}(x_{i}) 
	\end{array} \!\!\right]
	\!\!\in \mathbb{R}^{N \times n_{L}},
\end{equation}
where $\otimes$ denotes the Kronecker product of matrices, and 
$N = {\sum\limits}_{l=1}^{L} n_{l-1} \cdot n_{l}$ is the total number
of variables in the DNN.
Let us define \vspace{-1mm}
\begin{equation}
\label{eq:jacobi}
	\!\mathbf{P}(\mathbf{W}) \!:=\!\!
	\big[\! \operatorname{D}_{1}\!F(\mathbf{W}\!,x_{1}), \ldots,
	\!\operatorname{D}_{1}\!F(\mathbf{W}\!,x_{T}) \big]
	\!\!\in\! \mathbb{R}^{N \times T n_{L}}\!,\!\! \vspace{-1mm}
\end{equation}
and \vspace{-1mm}
%
\begin{equation}
	\bm{\varepsilon}(\mathbf{W}) \!:=\! 
	\big[ \nabla_{\!E}(\phi_{L}(x_{1}))^{\!\top}\!\!, \ldots, \!
	\nabla_{\!E}(\phi_{L}(x_{T}))^{\!\top} \big]^{\!\top}
	\!\!\!\in\! \mathbb{R}^{T n_{L}}\!,\!
\end{equation}
where $\nabla_{\!E}(\phi_{L}(x_{i})) \in \mathbb{R}^{n_{L}}$
denotes the gradient of $E(\cdot, \cdot)$ with respect to its first argument.
Then the critical point condition of the total loss function 
$\mathcal{J}$ can be presented as the following parameterised 
equation system in $\mathbf{W}$
\begin{equation}
	\nabla_{\!\mathcal{J}}(\mathbf{W}) :=
	\mathbf{P}(\mathbf{W}) \bm{\varepsilon}(\mathbf{W})
	= 0.
\end{equation}
%
%
Clearly, if there is no solution in $\bm{\mathcal{W}}$ for a
given finite set of samples, then the empirical total loss
function $\mathcal{J}$ has no critical points.
Since the error function $E$ is assumed to have global minima 
according to Assumption~\ref{ass:error}, i.e., the total loss
function $\mathcal{J}$ has a finite lower bound, there must be
a finite accumulation point.
On the other hand, if the trivial solution $\bm{\varepsilon}(\mathbf{W}) = 0$ 
is reachable at some weights $\mathbf{W}^{*} \in \bm{\mathcal{W}}$, then 
an exact DNN approximator $F(\mathbf{W}^{*}, \cdot)$ is obtained, i.e., 
$F(\mathbf{W}^{*}, x_{i}) = f^{*}(x_{i})$ by Assumption~\ref{ass:error}.
Furthermore, if the solution $\bm{\varepsilon}(\mathbf{W}) = 0$ is even
the only solution of the
parameterised linear equation system for all $\mathbf{W} \in \bm{\mathcal{W}}$, 
then any critical point of the loss function $\mathcal{J}$ is a global minimum.
Thus, we conclude the following theorem.

\begin{theorem}
\label{thm:optim}
	Given a DNN architecture $\bm{\mathcal{F}}(n_{0},\ldots,n_{L})$, and
	let the error function $E$ satisfy Assumption~\ref{ass:error}.
	If the rank of matrix $\mathbf{P}(\mathbf{W})$ as constructed in \eqref{eq:jacobi}
	is equal to $T n_{L}$ for all $\mathbf{W} \in \bm{\mathcal{W}}$, \vspace{-1mm}
	then 
	\begin{enumerate}[(1)]
		\item If exact learning of finite samples is \emph{achievable},
			i.e., $F(\mathbf{W}^{*}, x_{i}) = f^{*}(x_{i})$ 
			for all $i = 1, \ldots,T$, then $\mathbf{W}^{*}$ is a
			global minimum, and all
			critical points of $\mathcal{J}$ are global minima; \vspace{-1mm}
		\item If exact learning is \emph{unachievable}, then 
			the total loss $\mathcal{J}$ has no critical point, i.e.,
			the \emph{loss function} $\mathcal{J}$ is \emph{non-coercive} 
			\cite{guel:book10}. \vspace{1mm}
	\end{enumerate}
\end{theorem}
\begin{remark}
	Recent work \cite{aror:icml18} shows that over-parameterisation in 
	DNNs can accelerate optimisation in training DNNs.
	Such an observation can be explained by the results 
	in Proposition~\ref{thm:optim}, since for both exact and
	inexact learning, over-parameterisation enables exemption of
	both saddle points and suboptimal local minima.
	Note, that it is still a challenge to fully identify conditions 
	to ensure full rankness of $\mathbf{P}(\mathbf{W})$.
	Nevertheless, analysis in \cite{shen:cvpr18} suggest that
	making all weight matrices have full rank is a practical
	strategy to ensure the condition required in Theorem~\ref{thm:optim}.
	In the rest of this section, we show that DNNs with full rank
	weights are natural configurations of practice.
\end{remark}

Let us extend the Frobenius norm of matrices to collections of matrices as 
for any $\mathbf{W}_{1}, \mathbf{W}_{2} \in \bm{\mathcal{W}}$
\begin{equation}
	\big\| \mathbf{W}_{1}-\mathbf{W}_{2} \big\|_{F}^{2} := \sum_{l=1}^{L}
	\big\| W_{1,l}-W_{2,l} \big\|_{F}^{2}.
\end{equation}
It is simply the \emph{``entrywise''} norm of collections of matrices
in the same sense of $\ell_{2}$
norm of vectors.
Without loss of generality, we assume that weight $\mathbf{W}_{1} := 
\{W_{1,1}, \ldots, W_{1,L}\} \in \bm{\mathcal{W}}$ has the largest
rank-deficiency, i.e., all weight matrices are singular.
Then, for arbitrary $\epsilon > 0$, there exists always a full rank
weight $\mathbf{W}_{2} \in \bm{\mathcal{W}}$, so that
\begin{equation}
	\big\| \mathbf{W}_{1}-\mathbf{W}_{2} \big\|_{F}^{2}
	\le \epsilon.
\end{equation}
Let us denote by $\|\cdot\|_{2}$ the \emph{$\ell_{2}$ norm} of vectors or
the \emph{spectral norm} of matrices.
We can then apply a generalised mean value theorem of multi\-variate functions 
to the DNN map $F(\mathbf{W},x)$, where $x$ is treated as a constant, as
\begin{equation}
	\big\| F(\mathbf{W}_{1},x) - F(\mathbf{W}_{2},x) \big\|_{2} 
	\le c \, \big\| \mathbf{W}_{1}-\mathbf{W}_{2} \big\|_{F},
\end{equation}
where $\| \operatorname{D}_{1}\!F(\mathbf{W},x) \|_{2} \le c$ denotes
the upper bound of the largest singular value of the Jacobian matrix
of the network map $F(\mathbf{W},x)$ with respect to the weight $\mathbf{W}$
as computed in Eq.~\eqref{eq:dnn_jaco}, i.e., the map $F(\mathbf{W}, x)$ is
Lipschitz in weight $\mathbf{W}$.
Straightforwardly, we conclude the following result from the relationship between
$\ell_{2}$ norm and $\ell_{\infty}$ norm of vectors, i.e.,
$\|x\|_{2} \le \sqrt{n} \, \|x\|_{\infty}$ for $x \in \mathbb{R}^{n}$.

\begin{proposition}
	Given a DNN architecture $\bm{\mathcal{F}}(n_{0},\ldots,n_{L})$, 
	for any rank-deficient weight $\mathbf{W}_{1} \in \bm{\mathcal{W}}$,
	there exists a full-rank weight $\mathbf{W}_{2} \in \bm{\mathcal{W}}$,
	such that for arbitrary $\epsilon > 0$, the following inequality 
	holds true for all $x \in \bm{\mathcal{X}}$
	\begin{equation}
		\big\| F(\mathbf{W}_{2},x) - F(\mathbf{W}_{1},x) \big\|_{\infty} 
		\le \epsilon.
	\end{equation}
\end{proposition}
\begin{remark}
	This proposition ensures the existence of a DNN with full rank 
	weight matrices to approximate any weight configuration at arbitrary
	accuracy. 
	%
	In what follows, we show that the theory of differential
	topology is a natural theoretical framework for analysing DNNs by requiring 
	full rank weights to the properties of DNNs, and
	further investigate the other two challenges using the
	instruments from differential topology.
\end{remark}

\section{Expressiveness of DNNs: Width vs depth}
\label{sec:03}
Most data studied in machine learning often share some
low-dimensional structure.
In this work, we endow the input space $\bm{\mathcal{X}}$ with a smooth manifold structure.
\begin{assumption}
	The input space $\bm{\mathcal{X}} \subset \mathbb{R}^{n_{0}}$ is  
	a $m$-dimensional compact differentiable manifold 
	with $m \le n_{0}$.
\end{assumption}
\noindent Strictly speaking, a \emph{manifold} $\bm{\mathcal{X}}$ is a topological 
space that can \emph{locally} be continuously mapped to some vector space, 
where this map has a continuous inverse.
Namely, given any point $x \in U_{x} \subset \bm{\mathcal{X}}$, 
where $U_{x}$ is an open neighbourhood around $x$, there is an invertible map
%
	$\alpha \colon U_{x} \to \mathbb{R}^{k}$.
%
These maps are called \emph{charts}, and since charts are 
invertible, we can consider the change of two charts around 
any point in $\bm{\mathcal{X}}$ as a local map from the 
linear space into itself. 
If these maps are smooth for all points in $\bm{\mathcal{X}}$,
then $\bm{\mathcal{X}}$ is a \emph{smooth} 
manifold.
Trivially, the Euclidean space $\mathbb{R}^{m}$ is by nature
a smooth manifold.
We refer to \cite{leej:book10,leej:book13} for details about
manifolds.

\subsection{Properties of layer-wise maps}
Now, let us consider the first layer-wise map 
$\Lambda_{1}(W_{1},\cdot) \colon \bm{\mathcal{X}} \to \mathbb{R}^{n_{1}}$ 
as constructed in Eq.~\eqref{eq:layer}, which is a smooth 
map of smooth manifolds.
Then the differential of $\Lambda_{1}(W_{1},x)$ at $x \in \bm{\mathcal{X}}$
evaluated in tangent direction $\xi \in T_{x} \bm{\mathcal{X}}$
is computed as
\begin{equation}
	\operatorname{D}_{2}\!\Lambda_{1}(W_{1},x)\;\!\xi = 
	\Sigma'_{1}(x) W_{1}^{\top} \xi.
\end{equation}
%
%
Here, all diagonal entries of $\Sigma'_{1}(x)$ are always positive by choosing
activation functions to be smooth and monotonically increasing.
Since all weight matrices are assumed to have full rank, it is
clear that the differential $\Lambda_{1}(W_{1},\cdot)$ is 
a full rank linear map.

\begin{proposition}[Submersion layer]
	Let $\Lambda_{1} \colon \bm{\mathcal{X}} \to \mathbb{R}^{n_{1}}$ 
	be a layer-wise map as constructed in Eq.~\eqref{eq:layer}.
	If $\operatorname{dim} \bm{\mathcal{X}} \ge n$ and the weight matrix
	$W_{1}$ has full rank, then
	map $\Lambda_{1}$ is a \emph{submersion}.
\end{proposition}
\begin{corollary}
\label{coro:submerse}
	Given a DNN architecture $\bm{\mathcal{F}}(n_{0},\ldots,n_{L})$, 
	if $\operatorname{dim} \bm{\mathcal{X}} \ge n_{1} \ge \ldots \ge n_{L}$, 
	and all weight matrices $\mathbf{W}$ have full rank,
	then $\bm{\mathcal{F}}(n_{0},\ldots,n_{L})$ is a set of submersions
	from $\bm{\mathcal{X}}$ to $\mathbb{R}^{n_{L}}$.
\end{corollary}
\begin{remark}
	For a given $\mathbf{W}$, we denote by $\operatorname{Im}(F(\mathbf{W}, \cdot))$
	the image of the DNN map $F(\mathbf{W}, \cdot)$ on $\bm{\mathcal{X}}$.
	Then, it is straightforward to claim that all points in  
	$\operatorname{Im}(F(\mathbf{W}, \cdot))$ are regular points.
	If $\operatorname{dim} > n_{L}$, then the pre-image 
	$F^{-1}(\mathbf{W}, y)$ with $y \in 
	\operatorname{Im}(F(\mathbf{W}, \cdot))$ is a submanifold in $\bm{\mathcal{X}}$.
	More interestingly, disconnected sets in 
	$\bm{\mathcal{X}}$ can be mapped to a connected set in 
	$\operatorname{Im}(F(\mathbf{W}, \cdot))$,
	since the map $F(\mathbf{W}, \cdot)$ is surjective
	from $\bm{\mathcal{X}}$ to $\operatorname{Im}(F(\mathbf{W}, \cdot))$.
	%
	%

	A recent work \cite{nguy:icml18} claims that for 
	a DNN architecture $\bm{\mathcal{F}}(n_{0},\ldots,n_{L})$ with
	$n_{0} \ge n_{1} \ge \ldots \ge n_{L}$ and $n_{0} > n_{L}$,
	every open and connected set $V \in \mathbb{R}^{n_{L}}$ has
	its pre-image $U \in \mathbb{R}^{n_{0}}$ to be also open and 
	connected.
	Such a statement seems to obviously conflict with our conclusions above.
	A closer look reveals that the DNNs studied in \cite{nguy:icml18} 
	map $\mathbb{R}^{n_{0}}$ to $\mathbb{R}^{n_{L}}$, i.e.,
	$F(\mathbf{W},\cdot) \colon \mathbb{R}^{n_{0}} \to 
	\mathbb{R}^{n_{L}}$, while machine learning tasks are commonly constrained to
	some subset $\bm{\mathcal{X}} \subset \mathbb{R}^{n_{0}}$, i.e.,
	$\widetilde{F}(\mathbf{W},\cdot) \colon \bm{\mathcal{X}} \to 
	\mathbb{R}^{n_{L}}$.
	Specifically, if $\bm{\mathcal{X}}$ be open and disconnected, 
	then the image of $\widetilde{F}(\mathbf{W},\cdot)$, i.e., $\operatorname{Im}(\widetilde{F}(\mathbf{W}, \cdot))
	\in \mathbb{R}^{n_{L}}$,
	can be open but connected. There is no chance to
	infer the connectivity of $\bm{\mathcal{X}}$ from the connectivity of its
	image under strict surjective map.
	%
\end{remark}

\noindent Similarly, we have the following properties for an expanding 
DNN structure.
\begin{proposition}[Immersion layer]
	Let $\Lambda_{1} \colon \bm{\mathcal{X}} \to \mathbb{R}^{n_{1}}$ 
	be a layer-wise map as constructed in Eq.~\eqref{eq:layer}.
	If $\operatorname{dim} \bm{\mathcal{X}} \le n$ and the weight 
	$W_{1}$ has full rank, then
	map $\Lambda_{1}$ is an immersion.
\end{proposition}

\begin{corollary}
\label{coro:immerse}
	Given a DNN architecture $\bm{\mathcal{F}}(n_{0},\ldots,n_{L})$, 
	if $\operatorname{dim} \bm{\mathcal{X}} \le n_{1} \le \ldots \le n_{L}$, 
	and all weight matrices $\mathbf{W}$ have full rank,
	then $\bm{\mathcal{F}}(n_{0},\ldots,n_{L})$ is a set of immersions
	from $\bm{\mathcal{X}}$ to $\mathbb{R}^{n_{L}}$.
\end{corollary}
\begin{remark}
	By the construction of DNNs, any immersive DNN map is
	proper, i.e., their inverse images of compact subsets 
	are compact.
	Hence, immersive DNN maps are indeed embeddings.
	By the following theorem, topological properties of
	the data manifold are preserved under DNN embeddings.
\end{remark}

\begin{theorem}
\label{thm:embedding}
	Let $f \colon \bm{\mathcal{X}} \to \bm{\mathcal{Y}}$ be an embedding
	of smooth manifolds. Then, $\operatorname{Im}(f)$ is a submanifold 
	of $\bm{\mathcal{Y}}$.
\end{theorem}


%

\begin{corollary}
\label{coro:diffeo}
	Given a DNN architecture $\bm{\mathcal{F}}(n_{0},\ldots,n_{L})$, 
	if $n_{0} = n_{1} = \ldots = n_{L}$, 
	and all weight matrices $\mathbf{W}$ have full rank,
	then $\bm{\mathcal{F}}(n_{0},\ldots,n_{L})$ is a set of diffeomorphisms.
\end{corollary}


\subsection{Expressibility by composition of smooth maps}
In the previous subsection, we present some basic pro\-perties of 
layer-wise map, and simple DNN architectures.
In this subsection, we investigate expressibility of
more sophisticated DNN architectures as composition of smooth maps.
We assume that the output space $\bm{\mathcal{Y}}$ is a smooth 
submanifold of $\mathbb{R}^{q}$, i.e., $\bm{\mathcal{Y}} \subset \mathbb{R}^{q}$.
%

\begin{lemma}
\label{lem:compos}
	Let $f \colon \bm{\mathcal{X}} \to \bm{\mathcal{Y}} \subset \mathbb{R}^{q}$ 
	be a continuous map of smooth manifolds.
	Given a surjective linear map $g \colon \mathbb{R}^{p} \to \mathbb{R}^{q}$,
	i.e., $p \ge q$, there exists a continuous function 
	$h \colon \bm{\mathcal{X}} \to \mathbb{R}^{p}$, such that 
	$f = g \circ h$.
\end{lemma}
\begin{proof}
	Since $g$ is a surjective linear map, there exists an inverse
	map $g^{-1} \colon \mathbb{R}^{q} \to \mathbb{R}^{p}$, so that
	$g~\circ~g^{-1} = \operatorname{id}$.
	Trivially, we have $f = g~\circ~g^{-1} \circ f$, and
	constructing the continuous function $h (:= g^{-1} \circ f) \colon 
	\bm{\mathcal{X}} \to \mathbb{R}^{p}$ concludes the proof.
\end{proof}


\begin{theorem}
\label{thm:uniapp}
	Let $f \colon \bm{\mathcal{X}} \to \bm{\mathcal{Y}} \subset \mathbb{R}^{q}$ 
	be a continuous map of smooth manifolds.
	Given a surjective linear map $g \colon \mathbb{R}^{p} \to \mathbb{R}^{q}$
	with $p \ge q$, and $\epsilon > 0$, if $p > 2 \operatorname{dim} 
	\bm{\mathcal{X}}$, then there 
	exists a smooth embedding $\widetilde{h} \colon \bm{\mathcal{X}} \to 
	\mathbb{R}^{p}$, so that the following inequality 
	holds true for a chosen norm and all 
	$x \in \bm{\mathcal{X}}$
	\begin{equation}
		\big\| f(x) - g \circ \widetilde{h}(x) \big\| < \epsilon.
	\end{equation}
\end{theorem}
\begin{proof}
	According to Lemma~\ref{lem:compos}, it is equivalent to showing that for a 
	continuous function $h \colon \bm{\mathcal{X}} \to \mathbb{R}^{p}$,
	there is a smooth embedding $\widetilde{h} \colon \bm{\mathcal{X}} \to 
	\mathbb{R}^{p}$ that satisfies
	\begin{equation}
		\big\| g \circ h(x) - g \circ \widetilde{h}(x) \big\|_{2} < \epsilon.
	\end{equation}
	Since $g$ is linear by construction, we have
	\begin{equation}
		\big\| g \circ h(x) - g \circ \widetilde{h}(x) \big\|_{2} 
		\le \sigma_{1}(g) \cdot \big\| h(x) - \widetilde{h}(x) \big\|_{2},
	\end{equation}
	where $\sigma_{1}(g)$ is the largest singular value of
	the corresponding matrix representation.
	The weak Whitney Embedding Theorem \cite{leej:book13} ensures that
	for any $\varepsilon > 0$, if $p > 2 \operatorname{dim}\bm{\mathcal{X}}$,
	then there exists a smooth embedding 
	$\widetilde{h} \colon \bm{\mathcal{X}} \to \mathbb{R}^{p}$ such that 
	for all $x \in \bm{\mathcal{X}}$, we have
	\begin{equation}
		\big\| h(x) - \widetilde{h}(x) \big\|_{2} < \varepsilon.
	\end{equation}
	%
%
	%
	The result follows from the relationship between different norms.
\end{proof}
\begin{remark}
	It is important to notice that the lower bound in the Whitney Embedding 
	Theorem, i.e., $q >2 \operatorname{dim} \bm{\mathcal{X}}$, is not tight.
	Hence, it only suggests that, regardless of the depth, an agnostically 
	safe width of DNNs to ensure good approximation is at least twice of the 
	dimension of the data manifold.
	For a width-bounded DNN with $n_{l} \le 2 \operatorname{dim}\bm{\mathcal{X}}$,
	there is no guarantee to approximate arbitrary functions
	on an arbitrary data manifold.
\end{remark}

\section{Generalisability of DNNs: Explicit vs implicit regularisation}
\label{sec:04}
%

%
%
Since the DNNs studied in this work are constructed as composition of smooth maps,
it is natural to bound its output
using a generalised mean value theorem of multivariate functions \cite{soko:tsp17}.
Let $U$ be convex and open in $\mathbb{R}^{n_{0}}$ and given 
$x,x' \in U$, we have
\begin{equation}
	\big\| F(\mathbf{W},x) - F(\mathbf{W},x') \big\|_{2} \le 
	c \big\| x - x' \big\|_{2},
\end{equation}
where $\|\operatorname{D}\!F_{2}(\mathbf{W},x)\|_{2} \le c$ for 
all $x \in U$, with the Jacobian matrix 
of $F(\mathbf{W},x)$ with respect to $x$ being computed as
\begin{equation}
	\operatorname{D}\!F_{2}(\mathbf{W},x) = 
	W_{L}^{\top} \Sigma'_{L-1}(x) W_{L-1}^{\top} \ldots 
	\Sigma'_{1}(x) W_{1}^{\top}.
\end{equation}
Although this is a natural choice of error bound, it is still insufficient to explain
the so-called \emph{implicit regularisation} mystery, i.e., DNNs trained 
without explicit regularisers still perform well enough \cite{zhan:iclr17}.
%

Since the spectral norm of matrices is a smooth function, it is
conceptually easy to argue that DNNs should generalise well.
Here, we investigate the change rate of the spectral norm of
the Jacobian matrix $\operatorname{D}\!F_{2}(\mathbf{W},x)$ under
displacement in $x$.
Let us assume that the Jacobian matrix $\operatorname{D}\!F_{2}(\mathbf{W},x)$
has a distinct largest singular value.
Then, we can compute the directional derivative of the spectral norm
of the Jacobian matrix of DNNs as 
\begin{equation}
\begin{split}
	\operatorname{D}\!\| \operatorname{D}\!F_{2}(\mathbf{W},x) \|_{2} h 
	= &~
	u^{\top} \Big( \operatorname{D}\!\big(W_{L}^{\top} \Sigma'_{L-1}(x) 
	W_{L-1}^{\top} \ldots \Sigma'_{1}(x) W_{1}^{\top}\big)h \Big) v \\
	= &~ u^{\!\top}\! \bigg(\! \sum\limits_{l=1}^{L-1} \! W_{L}^{\top} \ldots W_{l+1}^{\top}
	\big( \operatorname{D}\!\Sigma'_{l}(x)h \big) 
	W_{l}^{\top} \ldots 
	W_{1}^{\top} \!\bigg) v, 
\end{split}
\end{equation}
where $u \in S^{n_{L}-1}$ and $v \in S^{n_{0}-1}$ are the left and
right singular vectors associated to the largest singular value of
the Jacobian matrix at $x$.
Let us denote by $\sigma_{l}''(x,h) \in \mathbb{R}^{n_{l}}$ 
the vector of diagonal entries of $\operatorname{D}\!\Sigma'_{l}(x)h$.
A tedious but straightforward computation leads to 
\begin{equation}
\label{eq:monster}
\begin{split}
	\operatorname{D}\!\|\! \operatorname{D}\!F_{2}(\mathbf{W},x) \|_{2} h 
	= & \underbrace{u^{\!\top} \! W_{L}^{\top}\! \big[ \!\operatorname{ddiag}(\delta_{L-1,1}), 
	\ldots, \operatorname{ddiag}(\delta_{L-1,
	\prod_{i=1}^{L-1} \!n_{i}
	}) \big]}_{\hspace{-12mm}=: \zeta^{\top}} \cdot \!\! \\[-3mm]
	& \cdot\! \underbrace{\bigg(\!\sum\limits_{l=1}^{L-1} \! \dot{\sigma}_{1}(x) \! \otimes \ldots \sigma_{l}''(x,h) 
	\ldots \otimes \! \dot{\sigma}_{L-1}(x) \! \bigg)}_{=: \eta \,\in\, \mathbb{R}^{\prod_{i=1}^{L-1} n_{i}}}, 
\end{split}
\end{equation}
where $\Delta_{l} = [\delta_{l,1}, \ldots, \delta_{l,\prod_{i=1}^{l} n_{i}}] 
:= W_{l}^{\top} \Delta_{l-1}$ with $\Delta_{1} := \operatorname{ddiag}(W_{1}^{\top} v)$.
Here, $\operatorname{ddiag}(\cdot)$ puts a vector into a diagonal matrix.
\begin{remark}
	In Eq.~\eqref{eq:monster}, the derivative of the spectral norm of 
	the Jacobian matrix of
	DNNs is computed as an inner product of two vectors, where the 
	vector $\zeta$ is a constant for a given DNN, and the other $\eta$ is
	dependent on derivatives of the activation functions.
	Simple explicit regularisations, e.g. weight decay \cite{krog:nips92} and
	path-norm \cite{neys:nips15,neys:nips17}, can be 
	simply justified for minimising the entries of $\zeta$.
	Furthermore, by computing 
	\begin{equation}
		\!\sigma_{l}''(x,h) \!=\! \Sigma''_{l}(x) \!\underbrace{W_{l}^{\!\top} 
		\Sigma'_{l-1}(x) W_{l-1}^{\top}
		\ldots \Sigma'_{1}(x) W_{1}^{\top}\!}_{=:\,\Omega_{l}(x)} \!h, \!
	\end{equation}
	we observe that the matrix $\Omega_{l}(x)$ is simply a truncation
	of the Jacobian of the DNN $F(\mathbf{W}, x)$ with respect to the input $x$.
	Minimising the Frobenius norm of the Jacobian matrix, known as 
	the Jacobian regulariser in \cite{soko:tsp17,jaku:eccv18}, is indeed a 
	more sophisticated explicit regularisation.
	
	The second term $\eta$ is the Kronecker
	product of derivatives of the activation functions,
	which is often upper bounded by one, e.g. Sigmoid, SoftSign, and SoftPlus. 
	Namely, the spectral norm of the Jacobian matrix of DNNs
	can only change slowly, hence DNNs without explicit 
	regularisations shall generalise well.
	As a result, we argue that the slope of activation functions
	is an implicit regularisation for generalisability of DNNs.	
\end{remark}

\section{Architecture of DNNs: Representation learn\-ing}
\label{sec:05}
So far, the empirical success of DNNs has been mostly observed 
and studied in the scenario of representation learning, which aims to  
extract suitable representations of data to promote solutions to 
machine learning problems \cite{beng:pami13,lecu:nature15}.
In particular, DNNs are capable of automatically
learning representations that are insensitive or invariant to nuisances,
such as translations, rotations, and occlusions.

One potential theory to explain such a phenomenon is the information
bottleneck (IB) principle \cite{tish:itw15}.
The original idea believes that training of DNNs performs two 
distinct phases, namely, an initial fitting phase 
and a subsequent compression phase.
The tradeoff between the two phases is guided by  the IB principle. 
The work in \cite{achi:jmlr18} further argues that discarding task-irrelevant 
information is necessary for learning invariant representations that 
generalises well.
However, a criticising work \cite{saxe:iclr18} demonstrates that
no evident causal connection can be found between compression and generalisation.
More interestingly, an opposite opinion states that loss of information
is unnecessarily responsible for generalisability of DNNs 
\cite{jaco:iclr18}.
Therefore, the IB theory of DNNs still needs a careful thorough investigation.

In this section, we propose to employ the quotient topo\-logy in the framework 
of differential topology, to model nuisance factors as 
equivalence relationship in data.
We refer to \cite{leej:book10} for details about quotient topology.
%
%
\begin{definition}[Nuisance as equivalence relation]
	Let $\bm{\mathcal{X}}$ be a data mani\-fold, 
	\emph{nuisance} on $\bm{\mathcal{X}}$ is defined
	as an equivalence relation $\sim$ on $\bm{\mathcal{X}}$.
\end{definition}
\noindent 
In the framework of differential topology, insensitivity or 
invariance to nuisances in data for a specific learning task
leads to the following assumption about the task map.

\begin{assumption}
\label{ass:task}
	The task map $f \colon \bm{\mathcal{X}} \to \bm{\mathcal{Y}}$
	is a surjective continuous map, i.e.,
	$f$ is invariant with respect to some nuisance/equivalence 
	relation $\sim$.
\end{assumption}
\noindent Then, we can define the nuisance relation $\sim$ on
$\bm{\mathcal{X}}$ by $x \sim x'$, if $f(x) = f(x')$, and
equivalence classes under $\sim$ on $\bm{\mathcal{X}}$ 
as 
%
	$[x] := \big\{ x' \in \bm{\mathcal{X}}\,|\, x \sim x' \big\}$,
which is also the fibre of $f$.
%
The set of equivalence class of $f$ is constructed as
%
	$\bm{\mathcal{X}}^{*} := \big\{ [x]~|~ x \in \bm{\mathcal{X}}
	\big\}$,
%
and endow $\bm{\mathcal{X}}^{*}$ the quotient topology via the 
canonical quotient map $\pi \colon \bm{\mathcal{X}} \to 
\bm{\mathcal{X}}^{*}$. 
%
%
Deep representation learning can be described
as a process of constructing suitable representation or feature
space $\bm{\mathcal{Z}}$ via $h \colon \bm{\mathcal{X}} \to \bm{\mathcal{Z}}$,
to enable a composition $f = g \circ h$ with
$g \colon \bm{\mathcal{Z}} \to \bm{\mathcal{Y}}$.
It can be visualised as 
the following commutative diagram
\begin{equation}
	\xymatrix{ 
	\bm{\mathcal{X}} \ar[r]^{f} \ar[d]_{h} & \bm{\mathcal{Y}}
	\\ \bm{\mathcal{Z}} \ar[ur]_{g} & }. 
\end{equation}
In this model, we refer to 
$h \colon \bm{\mathcal{X}} \to \bm{\mathcal{Z}}$ 
as the \emph{representation map} or \emph{feature map}, and 
$g \colon \bm{\mathcal{Z}} \to \bm{\mathcal{Y}}$
as the \emph{latent map}.

\begin{definition}[Sufficient representation]
	Let $\bm{\mathcal{X}}$ be a data mani\-fold, and
	$f \colon \bm{\mathcal{X}} \to \bm{\mathcal{Y}}$
	be a task function.
	A feature map $h \colon \bm{\mathcal{X}} 
	\to \bm{\mathcal{Z}}$ is \emph{sufficient} for the 
	task $f$, if there exists a function $g \colon
	\bm{\mathcal{Z}} \to \bm{\mathcal{Y}}$, so that
	$f = g \circ h$.
\end{definition}
\noindent Obviously, there can be an infinite number 
of possible constructions of representations.
In this work, we focus on two specific categories of
representations, namely, \emph{information-lossless representation}
and \emph{invariant representation}.

\subsection{Information-lossless representation}
The work in \cite{jaco:iclr18} constructs a cascade of invertible 
layers in DNNs, so that no information is discarded in the 
representations.
It shows that loss of information is not a necessary condition 
to learn representations that generalise well.
A similar observation is also made in the invertible 
convolutional neural networks \cite{glib:ijcai17}.
Instead of manually designing the invertible layers, we show that 
invertibility of layers in DNNs is its native properties, 
when the architecture of layers is suitable.

\begin{lemma}
\label{lem:embedding}
	Let $f \colon \bm{\mathcal{X}} \to \bm{\mathcal{Y}}$ be 
	a map of smooth manifolds. 
	Then $f$ can be decomposed as $f = g \circ h$, where 
	$h \colon \bm{\mathcal{X}} \to \mathbb{R}^{p}$ with 
	$p \ge 2 \operatorname{dim} \bm{\mathcal{X}}$ is a smooth embedding.
\end{lemma}
\begin{proof}
	By the strong Whitney embedding theorem, every smooth $m$-manifold 
	admits a smooth embedding into $\mathbb{R}^{p}$ with $p \ge 2m$.
	Then the image of the embedding feature map $h$, denoted by 
	$\bm{\mathcal{Z}} := \operatorname{Im}(h)$,
	is a smooth submanifold of $\mathbb{R}^{p}$, see Theorem~\ref{thm:embedding}.
	The embedding $h$ induces
	a diffeomorphism between $\bm{\mathcal{X}}$ and $\bm{\mathcal{Z}}$,
	i.e., 
	\begin{equation}
		\hbar \colon \bm{\mathcal{X}} \to \bm{\mathcal{Z}},
		\qquad 
		x \mapsto h(x).
	\end{equation}
	By properties of smooth embeddings \cite{guil:book74}, 
	there exists a smooth inverse $\hbar^{-1} \colon \bm{\mathcal{Z}} \to
	\bm{\mathcal{X}}$, so that $\hbar^{-1} \circ h = \operatorname{id}$.
	Tri\-vially, we have $f = f \circ \hbar^{-1} \circ h$, and the proof 
	is concluded by defining $g := f \circ \hbar^{-1}$.
\end{proof}



\noindent The relationship between the task map $f$ and the latent map
$g$ can be described as follows.

\begin{proposition}
\label{prop:embd_quot2}
	Let $\bm{\mathcal{X}}$ and $\bm{\mathcal{Y}}$ be smooth 
	manifolds, and the task map $f \colon \bm{\mathcal{X}} \to \bm{\mathcal{Y}}$
	admit a decomposition $f = g \circ h$, where 
	$h \colon \bm{\mathcal{X}} \to \mathbb{R}^{p}$ with 
	$p \ge 2 \operatorname{dim} \bm{\mathcal{X}}$ is a smooth embedding.
	Then the task map $f$ is a quotient map, \emph{if
	and only if} the latent map $g$ is a quotient map.
\end{proposition}
\begin{proof}
	Since the feature map $h$ is a diffeomorphism, $h$ is also a quotient map
	by definition. 
	If $g$ is a quotient map, then the composition $f = g \circ h$, composing 
	two quotient maps, is also a quotient map.

	Conversely, let us assume that $f$ is a quotient map.
	Since $f$ is surjective, so is $g$ surjective.
	Then it is equivalent to showing that a set $U \in \bm{\mathcal{Y}}$
	is open in $\bm{\mathcal{Y}}$, \emph{if and only if} the pre-image
	$g^{-1}(U)$ is open in $\bm{\mathcal{Z}}$.
	
	Suppose $g^{-1}(U)$ is open in $\bm{\mathcal{Z}}$.
	Then the set $h^{-1}(g^{-1}(U))$ is open in $\bm{\mathcal{X}}$ since $h$ 
	is a diffeomorphism.
	By assumption that $f = g \circ h$, i.e., $h^{-1}(g^{-1}(U)) = f^{-1}(U)$,
	and $f$ is a quotient map, hence this makes $U$ to be open in 
	$\bm{\mathcal{Y}}$.
	Now, let us assume $U$ is open in $\bm{\mathcal{Y}}$.
	Clearly, the set $h(f^{-1}(U))$ is open in $\bm{\mathcal{Z}}$, since 
	$f$ is a quotient map and $h$ is a diffeomorphism.
	The result follows from the fact that $g^{-1}(U) = h(f^{-1}(U))$.
\end{proof}


%
%

\subsection{Invariant representation}
Obviously, the dimension of information lossless representation can 
be large, hence the size of DNNs might explode.
It is thus demanding to construct lower dimensional representations
that can serve the same purpose.
In this subsection, we adopt the framework proposed in 
\cite{achi:jmlr18} to develop geometric notions of 
invariant representations.
Let us set the feature map to be 
the canonical quotient map, i.e., $h := \pi$ and 
$\bm{\mathcal{Z}} := \bm{\mathcal{X}}^{*}$.

\begin{definition}[Invariant representation]
	Let $\bm{\mathcal{X}}$ be a data mani\-fold, and
	$f \colon \bm{\mathcal{X}} \to \bm{\mathcal{Y}}$
	be a task map.
	A feature map $h \colon \bm{\mathcal{X}} 
	\to \bm{\mathcal{X}}^{*}$ is \emph{invariant} for the 
	task map $f$, if $f$ is constant on all pre-image
	$h^{-1}([x])$ with $[x] \in \bm{\mathcal{X}}^{*}$.
\end{definition}

\noindent
We then adopt the classic results about quotient maps \cite{leej:book10} to 
our scenario of invariant representation learning.

\begin{proposition}
	Let $\bm{\mathcal{X}}$ and $\bm{\mathcal{Y}}$ be smooth 
	manifolds, and the task map $f \colon \bm{\mathcal{X}} \to \bm{\mathcal{Y}}$
	satisfy Assumption~\ref{ass:task}.
	Then $f$ induces a unique bijective latent map 
	$g \colon \bm{\mathcal{X}}^{*} \to \bm{\mathcal{Y}}$
	such that $f = g \circ h$.
	Furthermore, the task map $f$ is a quotient map, \emph{if
	and only if} the latent map $g$ is a homeomorphism.
\end{proposition}

\noindent
Since a homeomorphic latent map $g$ implies the minimal dimension 
of the feature space $\bm{\mathcal{Z}}$, we conclude the following 
result.

\begin{corollary}
	Let $\bm{\mathcal{X}}$ and $\bm{\mathcal{Y}}$ be smooth 
	manifolds, and the task map $f \colon \bm{\mathcal{X}} \to \bm{\mathcal{Y}}$
	satisfy Assumption~\ref{ass:task}.
	If a feature map $h \colon \bm{\mathcal{X}} \to 
	\bm{\mathcal{X}}^{*}$ is both sufficient and 
	minimal,
	%
	then the task map $f$ is a quotient map.
\end{corollary}
%


\section{Experiments}
\label{sec:06}
In our experiments, all DNNs are trained in the batch learning setting.
The classic backpropagation algorithm and the approximate Newton's 
algorithm, proposed in \cite{shen:cvpr18}, are used for training 
DNNs.
Activation functions are all chosen to be Sigmoid.
The error function is a smooth approximation of the $\ell_{1}$ norm as 
$E(x,y):= \sqrt{\|x-y\|_{2}^{2} + \beta}$ with $\beta = 10^{-6}$.
%


\subsection{Learning as diffeomorphism}
In this experiment, we illustrate that the process of training DNNs 
is essentially deforming the data manifold diffeomorphically.
The task is the four region classification benchmark \cite{sing:nips89}.
In $\mathbb{R}^{2}$ around the origin, there is a square area $(-4,4)\times(-4,4)$, and 
three concentric circles with their radiuses being $1$, $2$, and $3$.
Four regions/classes are interlocked, nonconvex, as shown in Figure~\ref{fig:1a}.
\begin{figure*}[th!]
\centering
	\subfigure[$1000$ random samples]{\includegraphics[width=0.64\columnwidth]
		{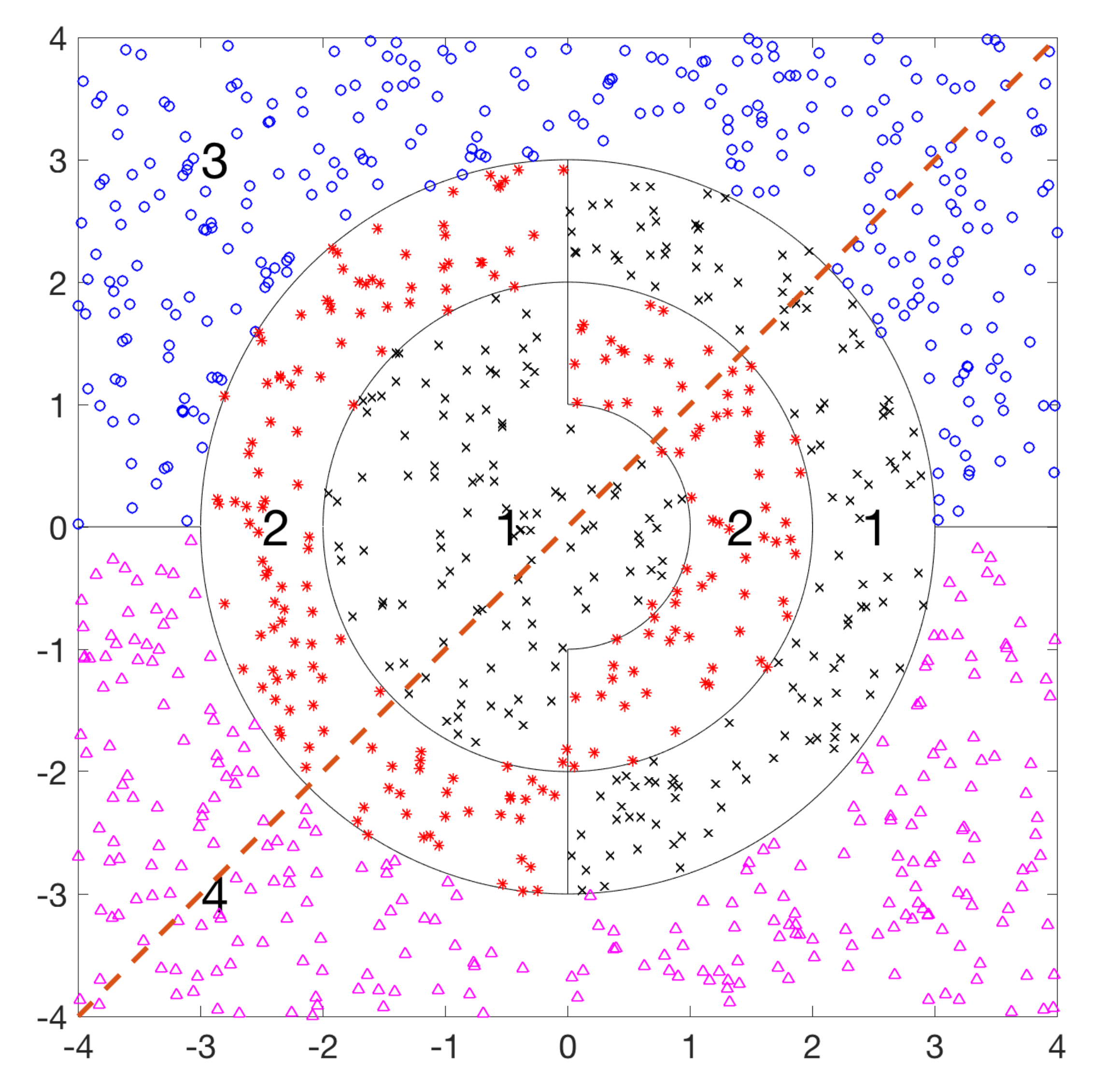}\label{fig:1a}} 
	\hspace{3mm}
	\subfigure[Tracking of outputs along the diagonal 
		(dash/solid: before/after convergence)]{\includegraphics[width=\columnwidth]
		{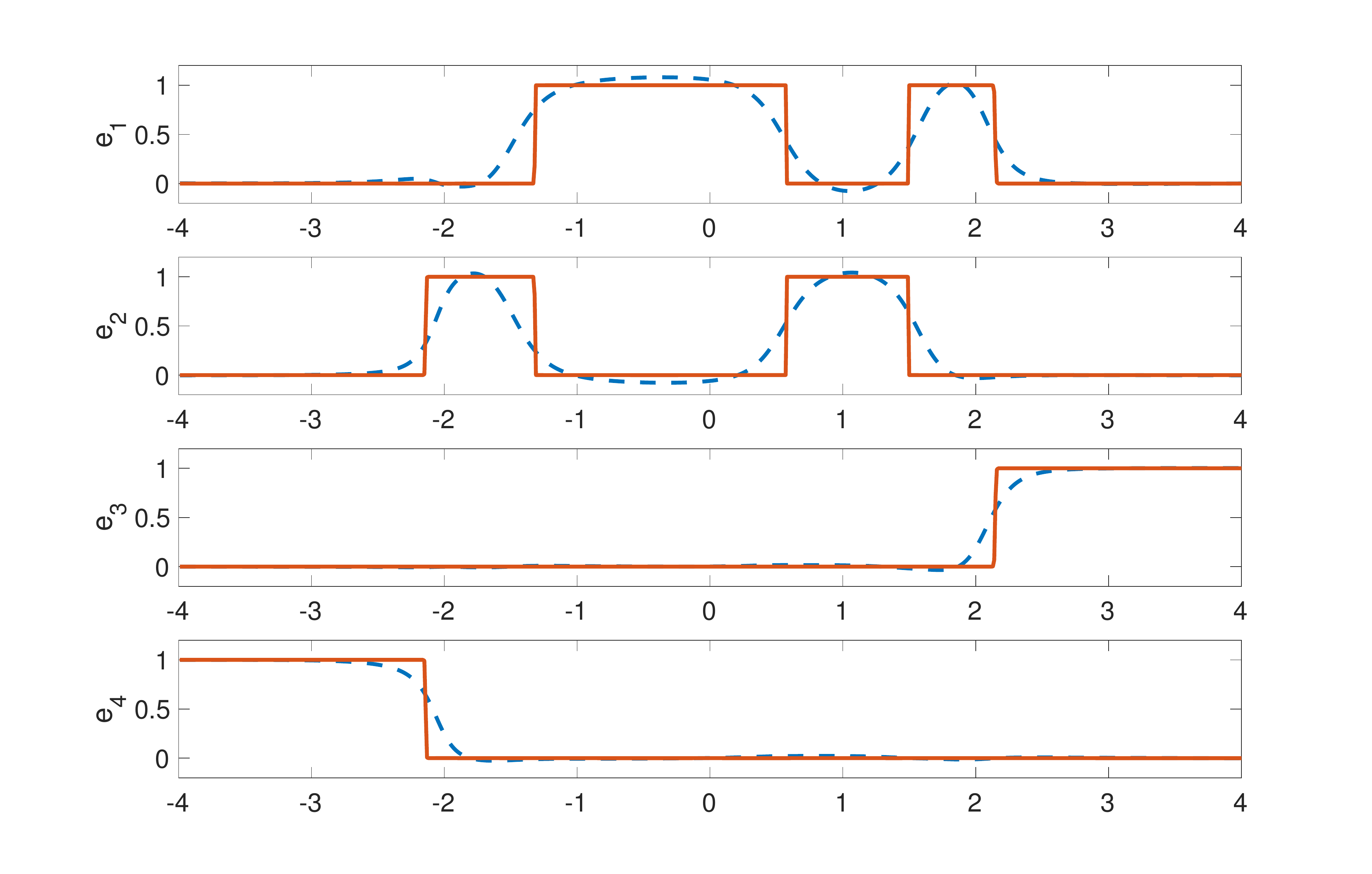}\label{fig:1b}} 
	\caption{Illustration of smooth deformation in the four regions classification 
	problem.}
	\label{fig:01} 
\end{figure*}
We randomly draw $T = 1000$ samples in the box for training, and specify the 
corresponding output to be the $i$-th basis vector in $\mathbb{R}^{4}$.
We deploy a four-layer DNN architecture $\bm{\mathcal{F}}(2,10,10,10,4)$.

We investigate the property of smoothly embedding
the $2D$ box/manifold into $\mathbb{R}^{4}$ via the specified DNN.
Since we cannot visualise a $4D$ structure, we track the values in 
each dimension of the output $\mathbb{R}^{4}$ along the diagonal (dashed) 
line (class transition $4\to2\to1\to2\to1\to3$, see Figure~\ref{fig:1a}).
Figure~\ref{fig:1b} shows two curves of DNN outputs along the diagonal 
in all four dimensions, where the dashed output curve (before convergence)
deforms smoothly to the final solid output curve (after convergence).
\begin{figure}[t!]
	\centering
	\subfigure[$a = 1$]{\includegraphics[width=0.75\columnwidth]
		{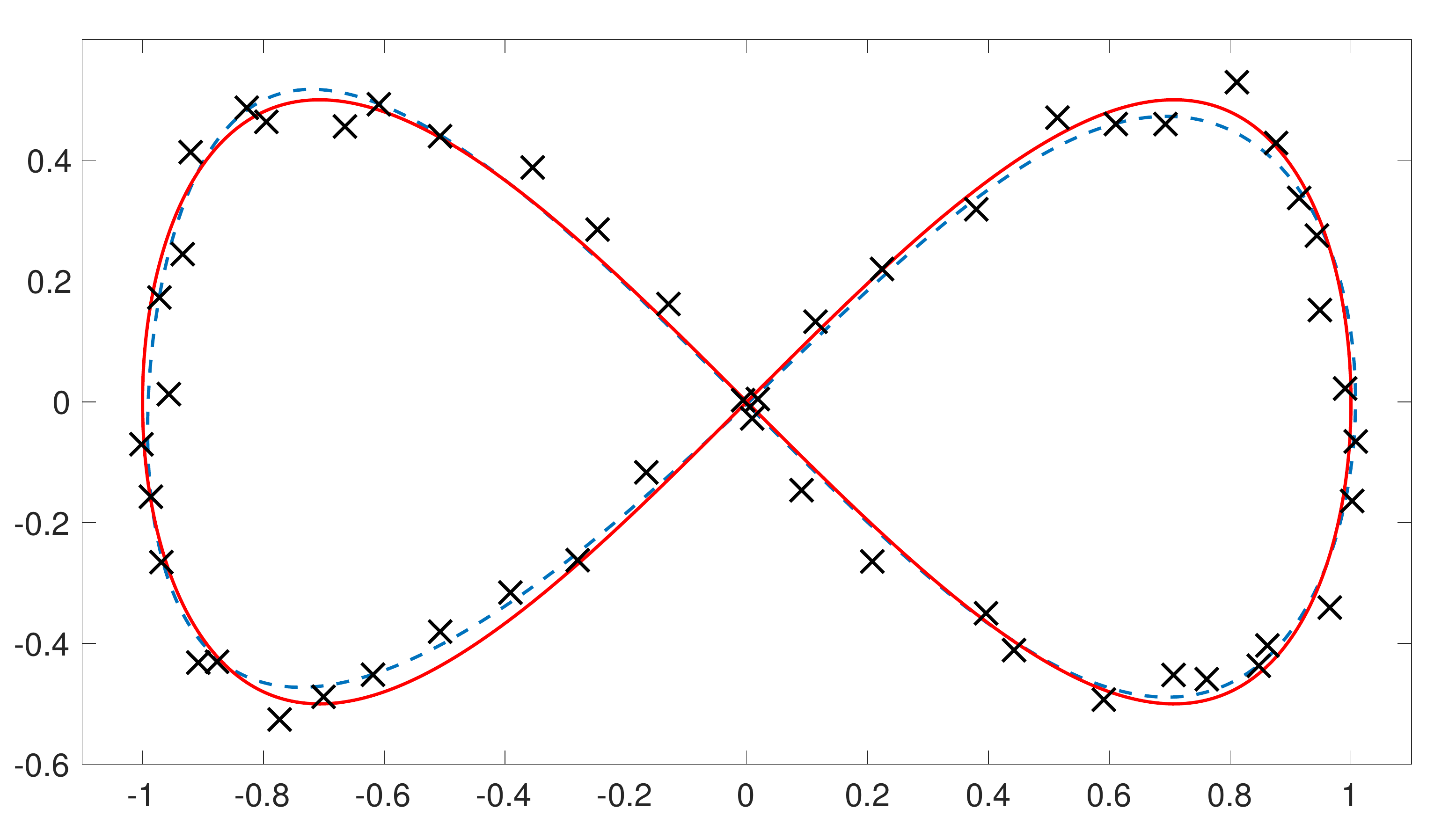}\label{fig:2a}} 
	\subfigure[$a = 5$]{\includegraphics[width=0.75\columnwidth]
		{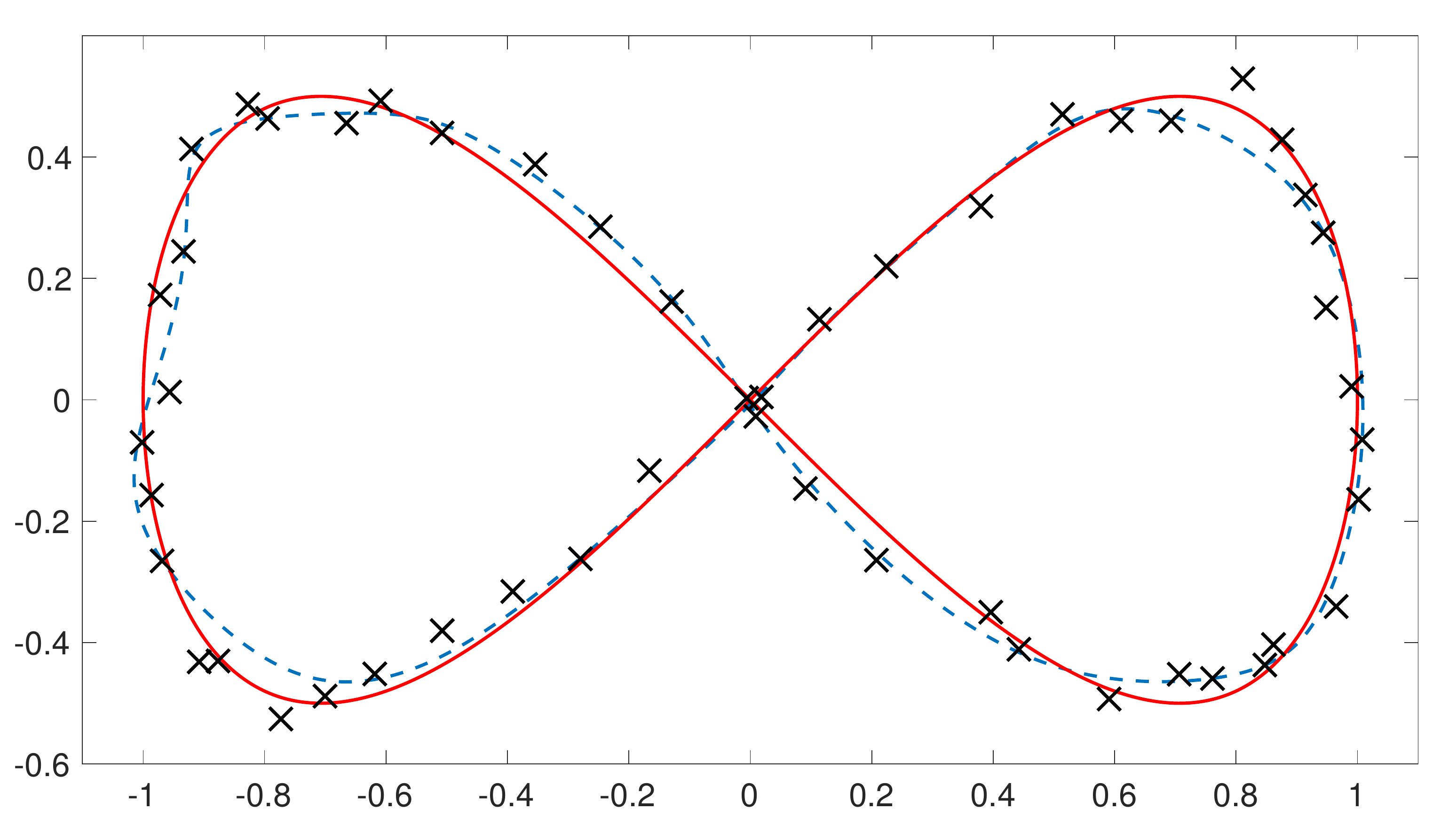}\label{fig:2b}}
	\subfigure[$a = 10$]{\includegraphics[width=0.75\columnwidth]
		{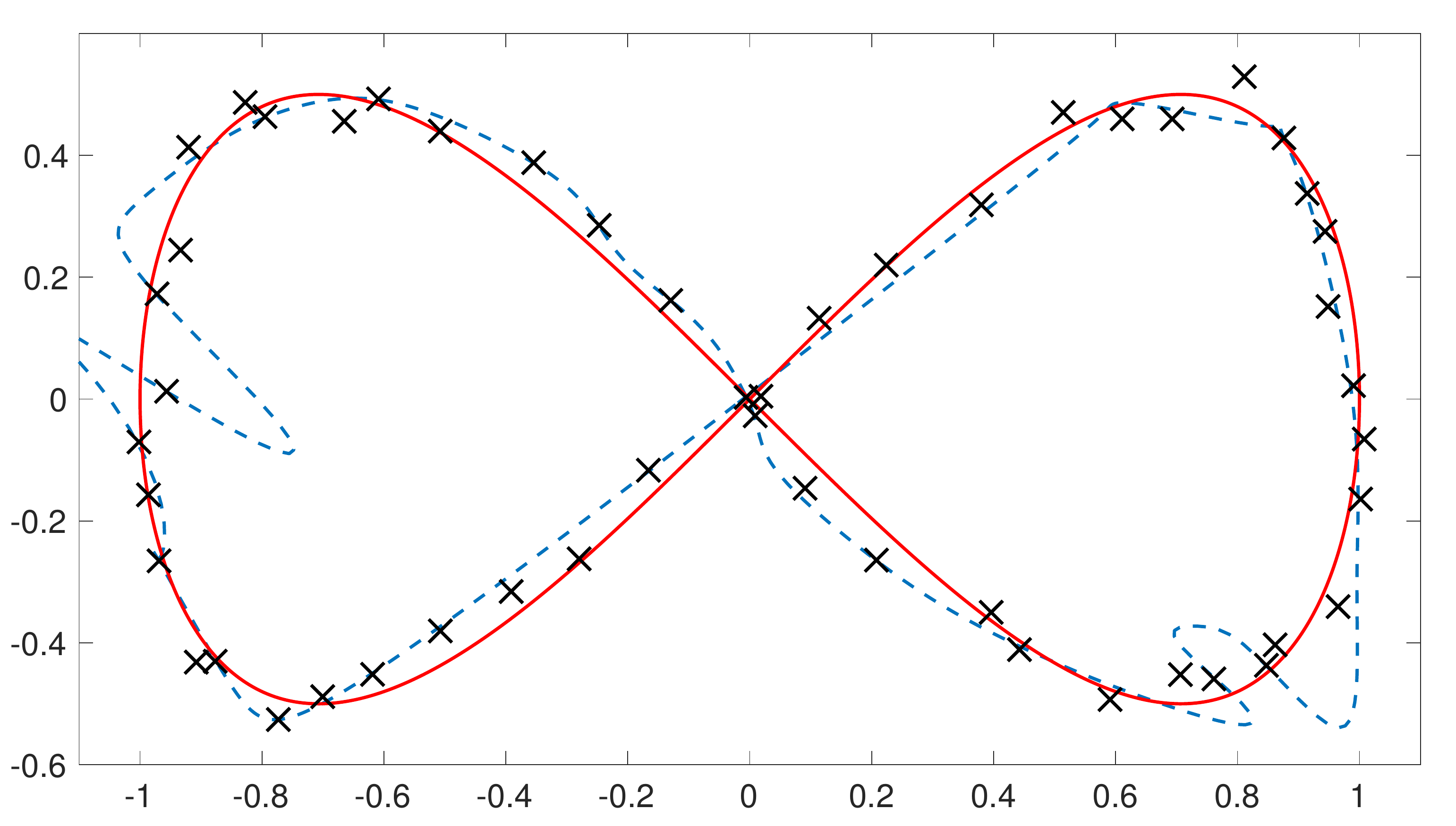}\label{fig:2c}}
	\caption{Implicit regularisation for generalisation (red solid: 
	ground truth figure eight; blue dashed: learned curve).}
	\vspace{-3mm}
	\label{fig:02} 
\end{figure}

\subsection{Implicit regularisation}
In this subsection, we investigate the results derived in Section~\ref{sec:04}
about regularisation for generalisation.
Since there have been enormous works about the effects of 
explicit regularisation, in this experiment we focus only on the 
implicit regularisation.

The task  is to learn a map from a unit circle 
$\bm{\mathcal{X}} := S^{1}$ to the two-petal rose 
$\bm{\mathcal{Y}} := R^{1}$, a.k.a. the figure eight curve.
Note, that the former is a smooth manifold, while the latter is not
a manifold due to the intersection at the origin.
We draw $51$ points equally placed on the circle, and perturb them
with a uniform noise in a square region $[-0.05, 0.05] \times [-0.05, 0.05]$.
A three-layer DNN architecture $\bm{\mathcal{F}}(2,10,10,2)$ is used.
The parameterised Sigmoid function and its derivative are defined 
as
\begin{equation}
	\sigma(x) \!:=\! \frac{1}{1+e^{-a x}}, \text{~~~and~}
	\sigma'(x) \!:=\! a\;\!\sigma(x) \big( 1 - \sigma(x) \big).
\end{equation}
In our experiment, we choose the constant $a \in \{1, 5, 10\}$,
which controls the largest slope of the Sigmoid function.
Figure~\ref{fig:02} depicts the learned curve against the ground truth,
and suggests that the performance of generalisation decreases with an 
increasing maximal slope $a$.
Clearly, large slopes of activation functions encourage overfitting.


\subsection{Deep representation learning}
In this experiment, we aim to investigate the findings in Section~\ref{sec:05}
about the DNN architecture in connection with generali\-sation.
The task is to map a Swiss roll (input manifold) with an arbitrary orientation to a 
unit circle $\bm{\mathcal{X}} := S^{1}$ (output manifold).
Specifically, we define the input manifold as
\begin{equation}
	\bm{\mathcal{X}}(Q) := \!\left\{\! Q \!\! \left.
	\begin{bmatrix}
		t \cos(t) \\
		t \sin(t) \\
		r
	\end{bmatrix}
	\right| t \in [0, 2\pi), r \in [-1,1]
	\right\},
\end{equation}
where $Q \in \mathbb{R}^{3 \times 3}$ is an arbitrary orthogonal matrix.
We randomly draw $T = 500$ samples on the Swiss roll
for training, and another $1000$ samples for testing.
We compare two DNN architectures, namely, one being a four-layer FNN 
$\bm{\mathcal{F}}(3,10,10,10,2)$ and the other being 
a five-layer FNN $\bm{\mathcal{F}}(3,10,10,1,10,2)$.
The former tends to learn information-lossless representations, while 
the later places a bottleneck to capture invariant representations.
The deep FNN has only one more neuron than the
shallow one.

We apply the two trained FNNs on $1000$ testing samples.
Figure~\ref{fig:03} shows the box plot of the $\ell_{2}$ norm
of prediction errors.
Clearly, the shallow FNN (left), which learns information-lossless 
representations, outperforms only slightly the deep FNN with a bottleneck (right),
in terms of mean value, variance, and tail.
However, we argue that such a difference is due to the difficulty 
in training the deep FNN with an extremely narrow bottleneck.

\begin{figure}[t!]
	\centering
	\includegraphics[width=\columnwidth]{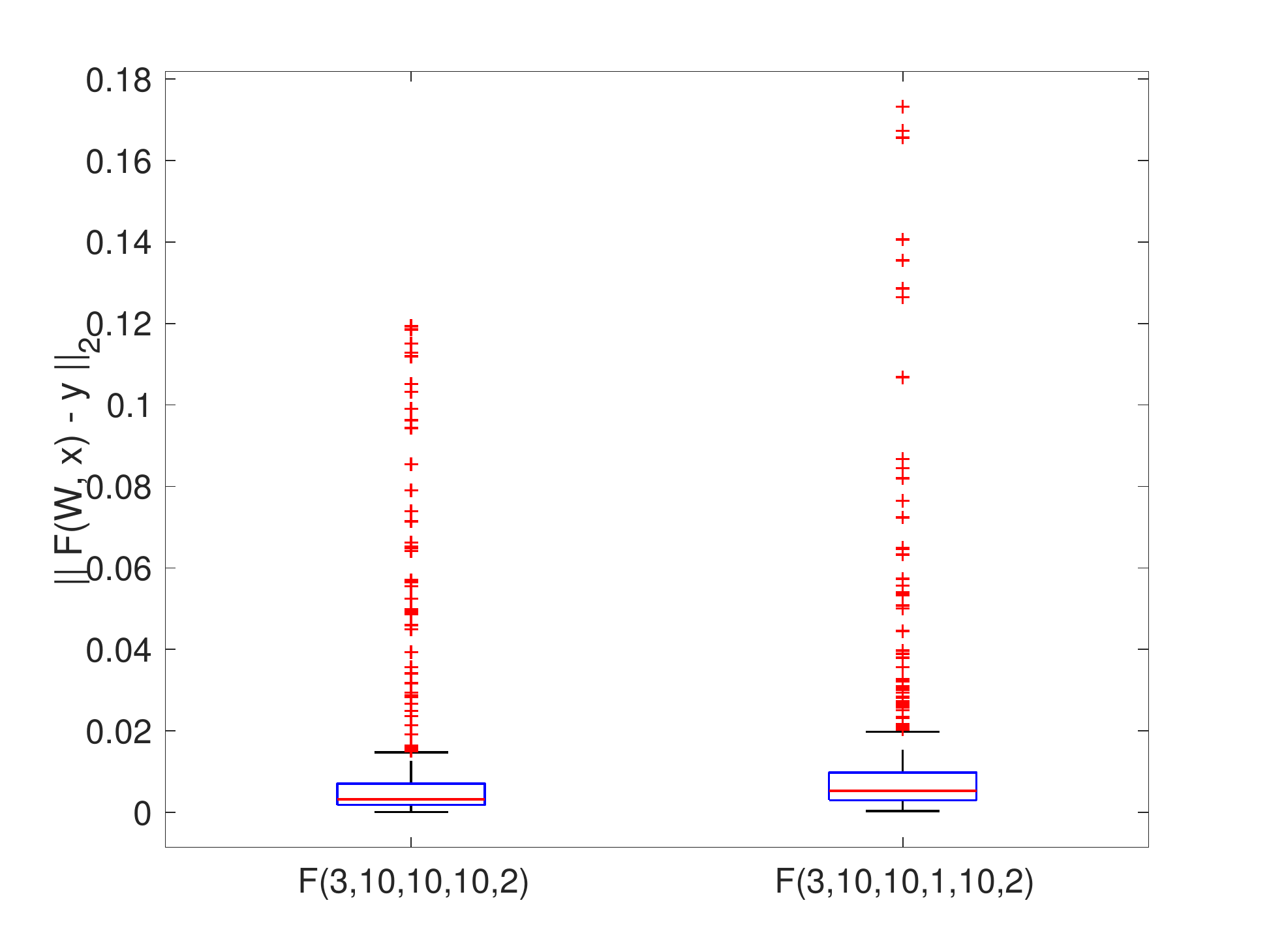}
	\caption{Box plot of generalisation errors (information lossless 
	representation vs invariant representation).}
	\label{fig:03} 
\end{figure}

\section{Conclusion}
\label{sec:07}
In this work, we provide a differential topological perspective on 
four challenging problems of learning with DNNs, namely,
\emph{expressi\-bility}, \emph{optimisability}, \emph{generalisability},
and \emph{architecture}.
By modelling the dataset of interest as a smooth manifold, DNNs 
are considered as compositions of smooth maps of smooth manifolds.
Our results suggest that differential topological instruments
are native for understanding and analysing DNNs.
We believe that a thorough investigation of differential topological
theory of DNNs will bring new knowledge and methodologies in
the study of deep learning.


\begin{thebibliography}{10}\itemsep=-1pt

\bibitem{achi:jmlr18}
A.~Achille and S.~Soatto.
\newblock Emergence of invariance and disentanglement in deep representations.
\newblock {\em Journal of Machine Learning Research}, 19(50):1--34, 2018.

\bibitem{aror:icml18}
S.~Arora, N.~Cohen, and E.~Hazan.
\newblock On the optimization of deep networks: Implicit acceleration by
  overparameterization.
\newblock In J.~Dy and A.~Krause, editors, {\em Proceedings of the $35^{th}$
  International Conference on Machine Learning}, Proceedings of Machine
  Learning Research, pages 244--253, 2018.

\bibitem{belk:icml18}
M.~Belkin, S.~Ma, and S.~Mandal.
\newblock To understand deep learning we need to understand kernel learning.
\newblock In {\em Proceedings of the $35^{th}$ International Conference on
  Machine Learning}, Proceedings of Machine Learning Research, pages 541--549,
  2018.

\bibitem{beng:pami13}
Y.~Bengio, A.~Courville, and P.~Vincent.
\newblock Representation learning: A review and new perspectives.
\newblock {\em IEEE Transactions on Pattern Analysis and Machine Intelligence},
  35(8):1789--1828, 2013.

\bibitem{beng:lskm07}
Y.~Bengio and Y.~LeCun.
\newblock Scaling learning algorithms toward {AI}.
\newblock In L.~Bottou, O.~Chapelle, D.~Decoste, and J.~Weston, editors, {\em
  Large-Scale Kernel Machines}, chapter~14, pages 321--358. MIT Press, 2007.

\bibitem{bish:book96}
C.~M. Bishop.
\newblock {\em Neural Networks for Pattern Recognition}.
\newblock Oxford University Press, USA, 1996.

\bibitem{elda:colt16}
R.~Eldan and O.~Shamir.
\newblock The power of depth for feedforward neural networks.
\newblock In {\em JMLR: Proceedings of Machine Learning Research: The 29$^{th}$
  Annual Conference on Learning Theory}, volume~49, pages 907--940, 2016.

\bibitem{fawz:cvpr18}
A.~Fawzi, S.-M. Moosavi-Dezfooli, P.~Forssard, and S.~Soatto.
\newblock Empirical study of the topology and geometry of deep networks.
\newblock In {\em Proceedings of the IEEE Conference on Computer Vision and
  Pattern Recognition (CVPR)}, pages 3762--3770, 2018.

\bibitem{glib:ijcai17}
A.~C. Gilbert, Y.~Zhang, K.~Lee, Y.~T. Zhang, and H.~Lee.
\newblock Towards understanding the invertibility of convolutional neural
  networks.
\newblock In {\em Proceedings of the Twenty-Sixth International Joint
  Conference on Artificial Intelligence (IJCAI-17)}, pages 1703--1710, 2017.

\bibitem{guil:book74}
V.~Guillemin and A.~Pollack.
\newblock {\em Differential Topology}, volume 370.
\newblock AMS Chelsea Publishing, 1974.

\bibitem{guel:book10}
O.~G{\"u}ler.
\newblock {\em Foundations of Optimization}.
\newblock Springer, 2010.

\bibitem{haef:cvpr17}
B.~D. Haeffele and R.~Vidal.
\newblock Global optimality in neural network training.
\newblock In {\em Proceedings of the IEEE Conference on Computer Vision and
  Pattern Recognition (CVPR)}, pages 7331 -- 7339, 2017.

\bibitem{haus:nips17}
M.~Hauser and A.~Ray.
\newblock Principles of riemannian geometry in neural networks.
\newblock In I.~Guyon, U.~V. Luxburg, S.~Bengio, H.~Wallach, R.~Fergus,
  S.~Vishwanathan, and R.~Garnett, editors, {\em Advances in Neural Information
  Processing Systems 30}, 2017.

\bibitem{horn:nn91}
K.~Hornik.
\newblock Approximation capabilities of multilayer feedforward networks.
\newblock {\em Neural Networks}, 4(2):251--257, 1991.

\bibitem{jaco:iclr18}
J.-H. Jacobsen, A.~Smeulders, and E.~Oyallon.
\newblock $i$-{R}ev{N}et: Deep invertible networks.
\newblock In {\em Proceedings of the $6^{th}$ International Conference on
  Learning Representations}, 2018.

\bibitem{jaku:eccv18}
D.~Jakubovitz and R.~Giryes.
\newblock Improving {DNN} robustness to adversarial attacks using {J}acobian
  regularization.
\newblock In V.~Ferrari, M.~Hebert, C.~Sminchisescu, and Y.~Weiss, editors,
  {\em European Conference on Computer Vision (ECCV)}, pages 525--541, 2018.

\bibitem{kawa:nips16}
K.~Kawaguchi.
\newblock Deep learning without poor local minima.
\newblock In D.~D. Lee, M.~Sugiyama, U.~V. Luxburg, I.~Guyon, and R.~Garnett,
  editors, {\em Advances in Neural Information Processing Systems 29}, pages
  586--594, 2016.

\bibitem{krog:nips92}
A.~Krogh and J.~A. Hertz.
\newblock A simple weight decay can improve generalization.
\newblock In J.~E. Moody, S.~J. Hanson, and R.~P. Lippmann, editors, {\em
  Advances in Neural Information Processing Systems 4}, pages 950--957, 1992.

\bibitem{lawr:aaai97}
S.~Lawrence, C.~L. Giles, and A.~C. Tsoi.
\newblock Lessons in neural network training: Overfitting may be harder than
  expected.
\newblock In {\em Proceedings of the $14^{th}$ National Conference on
  Artificial Intelligence and Ninth Conference on Innovative Applications of
  Artificial Intelligence}, AAAI'97/IAAI'97, pages 540--545, 1997.

\bibitem{lecu:nature15}
Y.~LeCun, Y.~Bengio, and G.~Hinton.
\newblock Deep learning.
\newblock {\em Nature}, 521:436--444, 2015.

\bibitem{leej:book10}
J.~M. Lee.
\newblock {\em Introduction to Topological Manifolds}.
\newblock Springer, $2^{nd}$ edition, 2010.

\bibitem{leej:book13}
J.~M. Lee.
\newblock {\em Introduction to Smooth Manifolds}.
\newblock Springer, $2^{nd}$ edition, 2013.

\bibitem{luzh:nips17}
Z.~Lu, H.~Pu, F.~Wang, Z.~Hu, and L.~Wang.
\newblock The expressive power of neural networks: A view from the width.
\newblock In I.~Guyon, U.~V. Luxburg, S.~Bengio, H.~Wallach, R.~Fergus,
  S.~Vishwanathan, and R.~Garnett, editors, {\em Advances in Neural Information
  Processing Systems 30}, pages 6231--6239, 2017.

\bibitem{mins:book17}
M.~Minsky and S.~A. Papert.
\newblock {\em Perceptrons: An Introduction to Computational Geometry}.
\newblock MIT Press, reissue of the 1988 expanded edition, 2017.

\bibitem{neys:nips17}
B.~Neyshabur, S.~Bhojanapalli, D.~McAllester, and N.~Srebro.
\newblock Exploring generalization in deep learning.
\newblock In I.~Guyon, U.~V. Luxburg, S.~Bengio, H.~Wallach, R.~Fergus,
  S.~Vishwanathan, and R.~Garnett, editors, {\em Advances in Neural Information
  Processing Systems 30}, pages 5947--5956, 2017.

\bibitem{neys:nips15}
B.~Neyshabur, R.~R. Salakhutdinov, and N.~Srebro.
\newblock Path-{SGD}: Path-normalized optimization in deep neural networks.
\newblock In C.~Cortes, N.~D. Lawrence, D.~D. Lee, M.~Sugiyama, and R.~Garnett,
  editors, {\em Advances in Neural Information Processing Systems 28}, pages
  2422--2430, 2015.

\bibitem{nguy:icml17}
Q.~Nguyen and M.~Hein.
\newblock The loss surface of deep and wide neural networks.
\newblock In {\em Proceedings of the $34^{th}$ International Conference on
  Machine Learning}, 2017.

\bibitem{nguy:icml18}
Q.~Nguyen, M.~Mukkamala, and M.~Hein.
\newblock Neural networks should be wide enough to learn disconnected decision
  regions.
\newblock In {\em Proceedings of the $35^{th}$ International Conference on
  Machine Learning}, 2018.

\bibitem{hart:book04}
H.~R. and A.~Zisserman.
\newblock {\em Multiple View Geometry in Computer Vision}.
\newblock Cambridge University Press, New York, 2004.

\bibitem{roln:iclr18}
D.~Rolnick and M.~Tegmark.
\newblock The power of deeper networks for expressing natural functions.
\newblock In {\em Proceedings of the $6$-th International Conference on
  Learning Representations}, 2018.

\bibitem{saxe:iclr18}
A.~M. Saxe, Y.~Bansal, J.~Dapello, M.~Advani, A.~Kolchinsky, B.~D. Tracey, and
  D.~D. Cox.
\newblock On the information bottleneck theory of deep learning.
\newblock In {\em Proceedings of the $6^{th}$ International Conference on
  Learning Representations}, 2018.

\bibitem{shen:cvpr18}
H.~Shen.
\newblock Towards a mathematical understanding of the difficulty in learning
  with feedforward neural networks.
\newblock In {\em Proceedings of the IEEE Conference on Computer Vision and
  Pattern Recognition (CVPR)}, pages 811--820, 2018.

\bibitem{sing:nips89}
S.~Singhal and L.~Wu.
\newblock Training multilayer perceptrons with the extended {K}alman algorithm.
\newblock In {\em Advances in Neural Information Processing Systems}, pages
  133--140, 1989.

\bibitem{soko:tsp17}
J.~Sokoli{\'c}, R.~Giryes, G.~Sapiro, and M.~R.~D. Rodrigues.
\newblock Robust large margin deep neural networks.
\newblock {\em {IEEE} Transactions on Signal Processing}, 65(16):4265--4280,
  2017.

\bibitem{sriv:jmlr14}
N.~Srivastava, G.~Hinton, A.~Krizhevsky, I.~Sutskever, and R.~Salakhutdinov.
\newblock Dropout: A simple way to prevent neural networks from overfitting.
\newblock {\em Journal of Machine Learning Research}, 15:1929--1958, 2014.

\bibitem{suns:aaai16}
S.~Sun, W.~Chen, L.~Wang, X.~Liu, and T.-Y. Liu.
\newblock On the depth of deep neural networks: A theoretical view.
\newblock In {\em Proceedings of the $30^{th}$ AAAI Conference on Artificial
  Intelligence}, pages 2066--2072, 2016.

\bibitem{sutt:ccss86}
R.~S. Sutton.
\newblock Two problems with backpropagation and other steepest-descent learning
  procedures for networks.
\newblock In {\em Proceedings of the $8$-th Annual Conference of the Cognitive
  Science Society}, pages 823--831, 1986.

\bibitem{telg:colt16}
M.~Telgarsky.
\newblock Benefits of depth in neural networks.
\newblock In {\em JMLR: Proceedings of Machine Learning Research: The 29$^{th}$
  Annual Conference on Learning Theory}, volume~49, 2016.

\bibitem{tish:itw15}
N.~Tishby and N.~Zaslavsky.
\newblock Deep learning and the information bottleneck principle.
\newblock In {\em IEEE Information Theory Workshop (ITW)}, 2015.

\bibitem{widr:pieee90}
B.~Widrow and M.~A. Lehr.
\newblock 30 years of adaptive neural networks: perceptron, madaline, and
  backpropagation.
\newblock {\em Proceedings of the IEEE}, 78(9):1415--1442, 1990.

\bibitem{yudo:book15}
D.~Yu and L.~Deng.
\newblock {\em Automatic Speech Recognition: A Deep Learning Approach}.
\newblock Springer-Verlag, London, 2015.

\bibitem{zhan:iclr17}
C.~Zhang, S.~Bengio, M.~Hardt, B.~Recht, and O.~Vinyals.
\newblock Understanding deep learning requires rethinking generalization.
\newblock In {\em The $5^{th}$ International Conference on Learning
  Representations}, 2017.

\end{thebibliography}

\end{document}